\documentclass{article}

\usepackage{tikz}
\usepackage{pgfplots}

\usepackage{times}
\usepackage{graphicx} 
\usepackage{subfigure} 

\usepackage{natbib}

\usepackage{algorithm}
\usepackage{algorithmic}

\usepackage{hyperref}
\usepackage[accepted]{icml2015}

\usepackage{amsmath}
\usepackage{amsfonts}
\usepackage{amsthm}
\usepackage{bm}
\usepackage{arydshln}

\newtheorem{theorem}{Theorem}
\newtheorem{proposition}{Proposition}

\newcommand{\x}{x}                   
\newcommand{\xspace}{{\cal X}}       
\newcommand{\y}{y}                   
\newcommand{\ystar}{\y^{*}}          
\newcommand{\Y}{{\cal Y}}            
\newcommand{\sample}{{\cal D}}       
\newcommand{\n}{n}                   
\newcommand{\N}{d}                   
\newcommand{\ppty}{p}                

\newcommand{\hypspace}{{\cal H}}     
\newcommand{\h}{h}                   
\newcommand{\hp}[1]{h_{#1}}	     
\newcommand{\w}{w}                   
\newcommand{\piexp}{h_0}             
\newcommand{\piexpp}[1]{h_0(#1)}     
\newcommand{\hcrm}{\hat{h}^{CRM}}    

\newcommand{\Loss}{{\Delta}}         
\newcommand{\loss}{{\delta}}         
\newcommand{\risk}[1]{R(#1)}
\newcommand{\riskemp}[1]{\hat{R}^M(#1)}
\newcommand{\riskempunclipped}[1]{\hat{R}(#1)}

\newcommand{\algname}{Policy Optimizer for Exponential Models}
\newcommand{\algnames}{POEM}  

\DeclareMathOperator*{\argmax}{argmax}
\DeclareMathOperator*{\argmin}{argmin}

\icmltitlerunning{Counterfactual risk minimization}

\begin{document} 

\twocolumn[
\icmltitle{Counterfactual Risk Minimization: Learning from Logged Bandit Feedback}

\icmlauthor{Adith Swaminathan}{adith@cs.cornell.edu}
\icmladdress{Cornell University,
            Ithaca, NY 14853 USA}
\icmlauthor{Thorsten Joachims}{tj@cs.cornell.edu}
\icmladdress{Cornell University,
            Ithaca, NY 14853 USA}

\icmlkeywords{empirical risk minimization, bandit feedback, importance sampling,
Bernstein bound, propensity, majorization}

\vskip 0.3in
]

\begin{abstract} 
We develop a learning principle and an efficient algorithm
for batch learning from logged bandit feedback.
This learning setting is ubiquitous in online systems
(e.g., ad placement, web search, recommendation),
where an algorithm makes a prediction (e.g., ad ranking)
for a given input (e.g., query)
and observes bandit feedback (e.g., user clicks on presented ads).
We first address the counterfactual nature of the learning problem through propensity scoring.
Next, we prove generalization error bounds that account
for the variance of the propensity-weighted empirical risk estimator.
These constructive bounds give rise to the Counterfactual Risk Minimization (CRM) principle.
We show how CRM can be used to derive a new
learning method -- called \algname\ (\algnames) -- for
learning stochastic linear rules for structured output prediction.
We present a decomposition of the \algnames\ objective that enables
efficient stochastic gradient optimization.
\algnames\ is evaluated on several multi-label classification problems
showing substantially improved robustness and generalization performance compared to the state-of-the-art.
\end{abstract}

\section{Introduction}
\label{sec:introduction}
Log data is one of the most ubiquitous forms of data available,
as it can be recorded from a variety of systems
(e.g., search engines, recommender systems, ad placement) at little cost.
The interaction logs of such systems typically contain a record of
the input to the system (e.g., features describing the user),
the prediction made by the system (e.g., a recommended list of news articles)
and the feedback (e.g., number of ranked articles the user read) \cite{Li2010}.
The feedback, however, provides only partial information -- ``bandit feedback''--
limited to the particular prediction shown by the system.
The feedback for all the other predictions the system could have made is typically not known.
This makes learning from log data fundamentally different from supervised learning,
where ``correct'' predictions (e.g., the best ranking of news articles for that user)
together with a loss function provide full-information feedback.

We study the problem of batch learning from logged bandit feedback.
Unlike online learning with bandit feedback,
batch learning does not require interactive experimental control over the system.
Furthermore, it enables the reuse of existing data and offline cross-validation techniques for model selection
(e.g., ``should we perform feature selection?'', ``which learning algorithm to use?'', etc.).

To solve this batch-learning problem, we first need a \emph{counterfactual} estimator \cite{Bottou2013}
of a system's performance, so that we can estimate
how other systems would have performed if they had been in control of choosing predictions.
Such estimators have been developed recently for the off-policy 
evaluation problem \cite{Dudik2011}, \cite{Li2011}, \cite{Li2014},
where data collected from the interaction logs of one bandit
algorithm is used to evaluate another system.

Our approach to counterfactual learning centers around the insight that,
to perform robust learning, it is not sufficient
to have just an unbiased estimator of the off-policy system's performance.
We must also reason about how the variances of these estimators
differ across the hypothesis space,
and pick the hypothesis that has the best possible guarantee
(tightest conservative bound) for its performance.
We first prove generalization error bounds analogous
to structural risk minimization \cite{Vapnik1998} for a \emph{stochastic hypothesis} family
using an empirical Bernstein argument \cite{Maurer2009}.
The constructive nature of these bounds suggests a general principle
-- Counterfactual Risk Minimization (CRM) --
for designing methods for batch learning from bandit feedback.

Using the CRM principle,
we derive a new learning algorithm -- \algname\ (\algnames) --
for structured output prediction.
The training objective is decomposed
using repeated variance linearization,
and optimizing it using AdaGrad \cite{Duchi2011} yields a fast and effective algorithm.
We evaluate \algnames\  on
several multi-label classification problems,
verify that its empirical performance supports the theory,
and demonstrate substantial improvement
in generalization performance over the state-of-the-art.

We review existing approaches in Section~\ref{sec:related}.
The learning setting is detailed in Section~\ref{sec:blsetting},
and contrasted with supervised learning.
In Section~\ref{sec:blcrm}, we derive the Counterfactual Risk Minimization learning principle
and provide a rule of thumb for setting hyper-parameters.
In Section~\ref{sec:blalgorithm}, we instantiate the CRM principle
for structured output prediction using exponential models
and construct an efficient decomposition of the objective for stochastic optimization.
Empirical evaluations are reported in Section~\ref{sec:experiments}
and we conclude with future directions and discussion in Section~\ref{sec:conclusion}.

\section{Related Work}
\label{sec:related}
Existing approaches for batch learning from
logged bandit feedback fall into two categories. The first
approach is to reduce the problem to supervised learning.
In principle, since the logs give us an incomplete view of the feedback for different predictions,
one could first use regression to estimate a feedback oracle for unseen predictions,
and then use any supervised learning algorithm using this feedback oracle.
Such a two-stage approach is known to not generalize well \cite{Beygelzimer2009}. 
More sophisticated techniques using a cost weighted classification \cite{Zadrozny2003}
or the Offset Tree algorithm \cite{Beygelzimer2009}
allow us to perform batch learning when the space of possible predictions is small.
In contrast, our approach generalizes structured output prediction,
with exponential-sized prediction spaces.

The second approach to batch learning from bandit feedback uses
propensity scoring \cite{Rosenbaum1983} to
derive unbiased estimators from the interaction logs \cite{Bottou2013}.
These estimators are used for a small set
of candidate policies, and the best estimated candidate is picked via exhaustive search.
In contrast, our approach can be optimized via gradient descent,
over hypothesis families (of infinite size)
that are equally as expressive as those used in supervised learning.

Our approach builds on counterfactual estimators
that have been developed for off-policy evaluation.
The inverse propensity scoring estimator can be optimal
when we have a good model of the historical algorithm \cite{Strehl2010}, \cite{Li2014}, \cite{Li2015},
and doubly robust estimators are even more efficient 
when we additionally have a good model of the feedback \cite{Dudik2011}.
In our work, we focus on the inverse propensity scoring estimator,
but the results we derive hold equally for the doubly robust estimators.
Recent work \cite{Thomas2015} has additionally developed
tighter confidence bounds for counterfactual estimators, which can be
directly co-opted in our approach to counterfactual learning.

In the current work, we concentrate on the case where
the historical algorithm was a stationary, stochastic policy.
Techniques like exploration scavenging \cite{Langford2008} and bootstrapping \cite{Mary2014}
allow us to perform counterfactual evaluation
even when the historical algorithm was deterministic or adaptive.

Our strategy of picking the hypothesis with the tightest conservative bound
on performance mimics similar successful approaches in other problems like
supervised learning \cite{Vapnik1998}, risk averse multi-armed bandits \cite{Galichet2013},
regret minimizing contextual bandits \cite{Langford2008b} and reinforcement learning \cite{Garcia2012}.

Beyond the problem of batch learning from bandit feedback,
our approach can have implications for several applications that
require learning from logged bandit feedback data:
warm-starting multi-armed bandits \cite{Shivaswamy2012} and contextual bandits \cite{Strehl2010},
pre-selecting retrieval functions for search engines \cite{Hofmann2013},
and policy evaluation for contextual bandits \cite{Li2011}, to name a few.

\section{Learning Setting: Batch Learning with Logged Bandit Feedback}
\label{sec:blsetting}
\begin{table*}[ht]
\caption{Comparison of assumptions, hypotheses and learning principles for supervised learning and batch learning with bandit feedback.}
\label{tab:SRMvsCRM}
\vskip -0.15in
\begin{center}
\begin{small}
\begin{tabular}{lccccl}
\hline
\abovespace\belowspace
Setting & Distribution & Data, $\sample$ & Hypothesis, $\h$ & Loss & Learning principle \\
\hline
\abovespace\belowspace
Supervised	& $(\x,\!\ystar)\!\sim\!\Pr(\xspace\!\times\!\Y)$ & $\{ \x_i,\!\ystar_i \}$ & $\y\!=\!\h(\x)$ & $\Loss(\ystar, \cdot)$ known & $\argmin_{\h} \riskempunclipped{\h} + C \cdot Reg(\hypspace)$ \\
\abovespace\belowspace
Batch w/bandit	& $\x\!\sim\!\Pr(\xspace), \y\!\sim\!\piexpp{\x}$ & $\{ \x_i,\!\y_i,\!\loss_i,\!\ppty_i \}$ & $\y\!\sim\!\h(\Y\!\mid\!\x)$ & $\loss(x, \cdot)$ unknown & $\argmin_{\h} \riskemp{\h} + \lambda \cdot \sqrt{\frac{{\bm{Var}(\h)}}{n}}$\\
\hline
\end{tabular}
\end{small}
\end{center}
\vspace{-0.2in}
\end{table*}
Consider a structured output prediction problem that
takes as input $\x \in \xspace$ and outputs a prediction $\y \in \Y$.
For example, in multi-label document classification, $\x$ could be a news article
and $\y$ a bitvector indicating the labels assigned to this article.
The inputs are assumed drawn from a fixed but unknown distribution $\Pr(\xspace)$,
$\x \overset{i.i.d.}{\sim} \Pr(\xspace)$.
Consider a hypothesis space $\hypspace$ of \emph{stochastic policies}.
A hypothesis $\h(\Y \mid \x) \in \hypspace$ defines a probability distribution
over the output space $\Y$,
and the hypothesis makes predictions by \emph{sampling}, $\y \sim \h(\Y \mid \x)$.
Note that this definition also includes deterministic hypotheses, where
the distributions assign probability $1$ to a single $\y$.
For notational convenience, denote $\h(\Y \mid \x)$ by $\h(\x)$, and 
the probability assigned by $\h(\x)$ to $\y$ as $\h(\y \mid \x)$.

In interactive learning systems, we only observe feedback $\loss(\x, \y)$
for the $\y$ sampled from $\h(\x)$.
In this work, feedback $\loss : \xspace \times \Y \mapsto \mathbb{R}$ is
a cardinal loss that is only observed at the sampled data points.
Small values for $\loss(\x, \y)$ indicate user satisfaction with $\y$ for $\x$,
while large values indicate dissatisfaction.
The expected loss -- called risk -- of a hypothesis $\risk{\h}$ is defined as,
\begin{equation*}
 \risk{\h} = \mathbb{E}_{\x \sim \Pr(\xspace)} \mathbb{E}_{\y \sim \h(\x)} \left[ \loss(\x, \y) \right].
\end{equation*}
The goal of the system is to minimize risk, or equivalently, maximize expected user satisfaction.
The aim of learning is to find a hypothesis $\h \in \hypspace$ that has minimum risk.

We wish to re-use the interaction logs of these systems for batch learning.
Assume that its historical algorithm acted according to a \emph{stationary} policy $\piexpp{\x}$
(also called logging policy).
The data collected from this system is 
\begin{equation*}
\sample = \{ (\x_1,\y_1,\loss_1), \dots, (\x_\n,\y_\n,\loss_\n) \},
\end{equation*}
where $\y_i \sim \piexpp{\x_i}$ and $\loss_i \equiv \loss(\x_i, \y_i)$.

\paragraph{Sampling bias.} $\sample$ cannot be used to estimate $\risk{\h}$ for a new hypothesis $\h$
using the estimator typically used in supervised learning.
We ideally need either full information about $\loss(x_i, \cdot)$
or need samples $\y \sim \h(\x_i)$ to directly estimate $\risk{\h}$.
This explains why, in practice, model selection over a small set of candidate systems
is typically done via A/B tests, where the candidates are deployed to collect new data
sampled according to $\y \sim \h(\x)$ for each hypothesis $\h$.
A relative comparison of the assumptions, hypotheses, and principles used in supervised learning vs. our learning
setting is outlined in Table~\ref{tab:SRMvsCRM}.
Fundamentally, batch learning with bandit feedback is hard because $\sample$ is both
\emph{biased} (predictions favored by the historical algorithm will be over-represented) 
and \emph{incomplete} (feedback for other predictions will not be available) for learning.

\section{Learning Principle: Counterfactual Risk Minimization}
\label{sec:blcrm}
The distribution mismatch between $\piexp$ and any hypothesis $\h \in \hypspace$
can be addressed using importance sampling, which corrects the sampling bias as:
\begin{align*}
 \risk{\h} &= \mathbb{E}_{\x \sim \Pr(\xspace)} \mathbb{E}_{\y \sim \h(\x)} \left[ \loss(\x, \y) \right]\\
 &= \mathbb{E}_{\x \sim \Pr(\xspace)} \mathbb{E}_{\y \sim \piexpp{\x}} \left[ \loss(\x, \y) \frac{\h(\y \mid \x)}{\piexpp{\y \mid \x}} \right].
\end{align*}
This motivates the propensity scoring approach \cite{Rosenbaum1983}.
During the operation of the logging policy,
we keep track of the propensity, $\piexpp{\y\!\mid\!\x}$ of the historical system to generate $\y$ for $\x$.
From these propensity-augmented logs 
\begin{equation*}
\sample\! =\! \{ (\x_1,\!\y_1,\!\loss_1,\!\ppty_1), \dots, (\x_\n,\!\y_\n,\!\loss_\n,\!\ppty_\n) \},
\end{equation*}
where $\ppty_i\! \equiv\! \piexpp{\y_i \mid \x_i}$, we can derive an unbiased estimate
of $\risk{\h}$ via Monte Carlo approximation,
\begin{equation}
\label{eq:montecarlo}
 \riskempunclipped{\h} = \frac{1}{\n} \sum_{i=1}^\n \loss_i \frac{\h(\y_i \mid \x_i)}{\ppty_i}.
\end{equation}
At first thought, one may think that directly estimating $\riskempunclipped{\h}$
over $\h \in \hypspace$ and picking the empirical minimizer is a valid learning strategy. Unfortunately, there are several potential pitfalls.

First, this strategy is not invariant to additive transformations of the loss and will give degenerate results if the loss is not appropriately scaled.
In Section~\ref{sec:loss_scale}, we develop intuition for why this is so, and derive the optimal scaling of $\loss$.
For now, assume that $\forall \x, \forall \y, \loss(\x, \y) \in \left[ -1, 0 \right]$.

Second, this estimator has unbounded variance, since $\ppty_i \simeq 0$ in $\sample$
can cause $\riskempunclipped{\h}$ to be
arbitrarily far away from the true risk $\risk{\h}$.
This problem can be fixed by ``clipping'' the importance sampling weights \cite{Ionides2008}
\begin{align}
 R^M(\h) &= \mathbb{E}_{\x} \mathbb{E}_{\y \sim \piexpp{\x}}\!\left[ \loss(\x, \y) 
 \min\!\left\{\!M,\!\frac{\h(\y\!\mid\!\x)}{\piexpp{\y\!\mid\!\x}}\!\right\} \right]\nonumber,\\
 \riskemp{\h} &= \frac{1}{\n} \sum_{i=1}^\n \loss_i \min \left\{ M, \frac{\h(\y_i \mid \x_i)}{\ppty_i} \right\}\nonumber.
\end{align}
$M > 0$ is a hyper-parameter chosen to trade-off bias and variance in the estimate, where
smaller values of $M$ induce larger bias in the estimate.
Optimizing $\riskemp{\h}$ through exhaustive enumeration over $\hypspace$ yields
the Inverse Propensity Scoring (IPS) training objective \cite{Bottou2013}
\begin{equation}
\label{eq:ipsobj}
 \hat{\h}^{IPS} = \argmin_{\h \in \hypspace} \left\{ \riskemp{\h} \right\}.
\end{equation}
Third, importance sampling typically estimates $\riskemp{\h}$ of different
hypotheses $\h \in \hypspace$ with vastly different variances.
Consider two hypotheses $\h_1$ and $\h_2$,
where $\h_1$ is similar to $\piexp$, but where $\h_2$
samples predictions that were not well explored by $\piexp$.
Importance sampling gives us low-variance estimates for $\riskemp{\h_1}$,
but highly variable estimates for $\riskemp{\h_2}$.
Intuitively, if we can develop variance-sensitive confidence bounds
over the hypothesis space, optimizing a conservative confidence bound
should find a $\h$ whose $\risk{\h}$ will not be much worse, with high probability.

\paragraph{Generalization error bound.} A standard analysis would give
a bound that is agnostic to variance introduced by importance sampling.
Following our intuition above, we derive a higher order bound that includes the variance term
using empirical Bernstein bounds \cite{Maurer2009}.
To develop such a generalization error bound,
we first need a concept of capacity for stochastic hypothesis classes.
For any stochastic class $\hypspace$, define an auxiliary function class $\mathcal{F}_\hypspace\!=\!\{ f_\h\! :\! \xspace\!\times\!\Y\!\mapsto \left[ 0, 1 \right] \}$.
Each $\h \in \hypspace$ corresponds to a function $f_\h \in \mathcal{F}_\hypspace$,
\begin{equation}
\label{eq:translate1}
  f_\h(\x, \y) = 1 + \frac{\loss(\x, \y)}{M} \min \left\{ M, \frac{\h(\y \mid \x)}{\piexpp{\y \mid \x}} \right\}.
\end{equation}
$f_\h$ is a deterministic, bounded function, and satisfies
\begin{equation}
\label{eq:translate2}
\mathbb{E}_{\x} \mathbb{E}_{\y \sim \piexpp{\x}} \left[ f_\h(\x, \y) \right] = 1 + R^M(\h)/M .
\end{equation}
Hence, we can use classic notions of capacity for $\mathcal{F}_\hypspace$ to reason about the convergence of $\riskemp{\h} \rightarrow R^M(\h)$.

Recall the covering number $\mathcal{N}_\infty( \epsilon, \mathcal{F}, \n)$ for a function class $\mathcal{F}$ 
(refer \cite{Anthony2009}, \cite{Maurer2009} and the references therein).
Define an $\epsilon-$cover $\mathcal{N}(\epsilon, A, \| \cdot \|_\infty)$ for a set $A \subseteq \mathbb{R}^n$
to be the size of the smallest cardinality subset $A_0 \subseteq A$
such that $A$ is contained in the union of balls of radius $\epsilon$
centered at points in $A_0$, in the metric induced by $\| \cdot \| _ \infty$.
The covering number is,
 \begin{equation*}
  \mathcal{N}_\infty( \epsilon, \mathcal{F}, \n)\! = \! \sup_{(\x_i,\y_i) \in (\xspace \times \Y)^\n}
  \mathcal{N}(\epsilon, \mathcal{F}(\{(\x_i,\y_i)\}), \| \cdot \| _ \infty),
 \end{equation*}
where $\mathcal{F}(\{(\x_i,\y_i)\})$ is the function class conditioned on sample $\{(\x_i,\y_i)\}$,
\begin{equation*}
\mathcal{F}(\{(\x_i,\y_i)\}) = \{ (f(\x_1, \y_1), \dots, f(\x_\n, \y_\n)): f \in \mathcal{F} \}.
\end{equation*}
Our measure for the capacity of our stochastic class $\hypspace$
to ``fit'' a sample of size $\n$
shall be $\mathcal{N}_\infty(\frac{1}{\n}, \mathcal{F}_\hypspace, 2\n)$.

\begin{theorem}
\label{thm:genbound}
 For a compact notation, define
\begin{align*}
\vspace{-0.2in}
  &{u_{\h}}^i \equiv \loss_i \min\{ M, \h(\y_i \mid \x_i) / \ppty_i \}, \quad \overline{u_{\h}} \equiv \sum_{i=1}^\n {u_{\h}}^i / \n,\\
  &{\bm{Var}_{\h}}(u) \equiv \sum_{i=1}^\n ({u_{\h}}^i - \overline{u_{\h}})^2 / (n - 1),\\  
  &\mathcal{Q}_\hypspace(\n, \gamma) \equiv \log(10 \cdot \mathcal{N}_\infty( \frac{1}{n}, \mathcal{F}_\hypspace, 2\n) /\gamma), \quad 0 < \gamma <  1.
\end{align*}
 With probability at least $1 - \gamma$ in the random vector $(\x_1, \y_1) \cdots (\x_\n, \y_\n)$, with 
 $\x_i \overset{i.i.d.}{\sim} \Pr(\xspace)$ and $\y_i \sim \piexp(\x_i)$, and
 observed losses $\loss_1,\dots,\loss_\n$,
 for $\n \ge 16$ and a stochastic hypothesis space $\hypspace$ with capacity
 $\mathcal{N}_\infty(\frac{1}{\n}, \mathcal{F}_\hypspace, 2\n)$,
 \begin{align*}
  \forall \h \in \hypspace:
  \risk{\h} &\le \riskemp{\h} + \sqrt{18 {\bm{Var}_{\h}}(u) \mathcal{Q}_\hypspace(\n, \gamma) / \n}\\
  & + M \cdot 15 \mathcal{Q}(\n, \gamma) / (\n-1).
  \label{eq:genbound}
 \end{align*}
\end{theorem}

\begin{proof}
Follow the proof of Theorem 6 of \cite{Maurer2009} with the function class as $\mathcal{F}_\hypspace$.
Use Equations~\eqref{eq:translate1}, \eqref{eq:translate2} to translate from $f_\h(\x, \y)$ to $R^M(\h)$.
$\riskemp{\h} = M \cdot \hat{f}_\h -1$, $R^M(\h) = M \cdot f_\h -1$, and 
$M^2 {\bm{Var}_{\h}}(u) = {\bm{Var}_{f_\h}}(u)$.
Finally, since $\loss(\cdot,\!\cdot)\!\le\!0$, hence $\risk{\h}\! \le\! R^M(\h)$.
\end{proof}

\paragraph{CRM Principle.} This generalization error bound is constructive, and
it motivates a general principle for designing machine learning methods for batch learning from bandit feedback.
In particular, a learning algorithm following this principle
should jointly optimize the estimate $\riskemp{\h}$ as well as its empirical standard deviation,
where the latter serves as a \emph{data-dependent regularizer}.
\begin{equation}
\hcrm = \argmin_{\h \in \hypspace} \left\{ \riskemp{\h} + \lambda \sqrt{\frac{{\bm{Var}_{\h}}(u)}{\n}} \right\} \label{eq:crmobj}.
\end{equation}
$M > 0$ and  $\lambda \ge 0$ are regularization hyper-parameters.
When $\lambda = 0$, we recover the Inverse Propensity Scoring objective of Equation~\eqref{eq:ipsobj}.
In analogy to Structural Risk Minimization \cite{Vapnik1998},
we call this principle {\em Counterfactual Risk Minimization},
since both pick the hypothesis with the tightest upper bound on the true risk $\risk{\h}$.

\subsection{Optimal Loss Scaling}
\label{sec:loss_scale}
When performing supervised learning with true labels $\y^*$
and a loss function $\Loss(\ystar, \cdot)$,
empirical risk minimization using the standard estimator
is invariant to additive translation and multiplicative scaling of $\Loss$.
The risk estimators $\riskempunclipped{\h}$ and $\riskemp{\h}$ in bandit learning,
however, crucially require $\loss(\cdot, \cdot) \in \left[ -1, 0 \right]$.

Consider, for example, the case of $\loss(\cdot, \cdot) \ge 0$.
The training objectives in Equation~\eqref{eq:ipsobj} (IPS)
and Equation~\eqref{eq:crmobj} (CRM)
become degenerate!
A hypothesis $\h \in \hypspace$ that completely avoids the sample $\sample$
(i.e. $\forall i = 1, \dots, \n, \h(\y_i \mid \x_i) = 0$)
 trivially achieves the best possible $\riskemp{\h}$ ($= 0$) with $0$ empirical variance.
This degeneracy arises because when $\loss(\cdot, \cdot) \ge 0$,
the optimization objectives are a \emph{lower} bound on $\risk{\h}$, whereas what we need is an \emph{upper} bound.

For any bounded loss $\delta(\cdot, \cdot) \in \left[ \bigtriangledown, \bigtriangleup \right]$,
we have, $\forall \x$
\begin{equation*}
\mathbb{E}_{\y \sim \h(\!\x\!)} \!\left[ \loss(\x, \y) \right]\! \le \! \bigtriangleup \! + \mathbb{E}_{\y \sim \piexpp{\!\x\!}} \! \left[ (\loss(\x, \y)\! -\! \bigtriangleup) \frac{\h(\y \mid \x)}{\piexpp{\y \mid \x}} \right].
\end{equation*}
We assert that this is the tightest possible upper bound possible without additional assumptions.
Since the optimization objectives in Equations~\eqref{eq:ipsobj},\eqref{eq:crmobj} are unaffected by a constant scale factor
(e.g., $\bigtriangleup - \bigtriangledown$), we should transform $\loss \mapsto \loss'$ to derive a conservative training objective w.r.t. $\loss'$,
\begin{equation*}
 \loss' \equiv \{\loss - \bigtriangleup\}/\{\bigtriangleup - \bigtriangledown\}.
\end{equation*}

\subsection{Selecting hyper-parameters}
\label{sec:hyperparam_scale}
We propose selecting the hyper-parameters $M > 0$ and $\lambda \ge 0$ via validation.
However, we must be careful not to set $M$ too small or $\lambda$ too big.
The estimated risk $\riskemp{\h} \in \left[ -M, 0 \right]$,
while the variance penalty $\sqrt{\frac{{\bm{Var}_{\h}}(u)}{\n}} \in \left[ 0, \frac{M}{2\sqrt{\n}} \right]$.
If $M$ is too small, all hypotheses will have the same biased estimate of risk $M \riskemp{\piexp}$, since all the importance sampling weights will be clipped.
Similarly, if $\lambda \gg 0$, a hypothesis $\h \in \hypspace$ that completely avoids $\sample$
achieves the best possible training objective of $0$.
As a rule of thumb, we can calibrate $M$ and $\lambda$ so that the estimator is unbiased and
objective is negative for some $\h \in \hypspace$.
When $\piexp \in \hypspace$, $M \simeq \max \{ \ppty_i \} / \min \{ \ppty_i \}$ and $\left\{ \riskemp{\piexp} + \lambda \sqrt{\frac{{\bm{Var}_{\piexp}}(u)}{\n}} \right\} < 0$ are natural choices.

\subsection{When is counterfactual learning possible?}
\label{sec:cflearning_feasible}
The bounds in Theorem~\ref{thm:genbound} are with respect to
the randomness in $\piexp$.
Known impossibility results for counterfactual evaluation using $\piexp$ \cite{Langford2008}
also apply to counterfactual learning.
In particular, if $\piexp$ was deterministic, or even stochastic but without
full support over $\Y$, it is easy to engineer examples involving the unexplored $y \in \Y$
that guarantee sub-optimal learning even as $\left| \sample \right| \rightarrow \infty$.
Also, a stochastic $\piexp$ with heavier tails need not always allow more effective learning.
From importance sampling theory \cite{Owen2013},
what really matters is how well $\piexp$ explores the regions of $\Y$ with favorable losses.

\section{Learning Algorithm: \algnames}
\label{sec:blalgorithm}
We now use the CRM principle to derive an efficient algorithm for structured output prediction using linear rules.
Classic models in supervised learning (e.g., structured support vector machines \cite{Tsochantaridis2004} and conditional random fields \cite{Lafferty2001}) predict using
\begin{equation}
\h^{sup}_{\w}(\x) = \argmax_{\y \in \Y} \left\{ \w \cdot \phi(\x, \y) \right\},
\label{eq:linearhyp}
\end{equation}
where $\w$ is a $\N-$dimensional weight vector, and $\phi(\x, \y)$ is a $\N-$dimensional joint feature map.
For example, in multi-label document classification, for a news article $x$
and a possible assignment of labels $\y$ represented as a bitvector,
$\phi(\x, \y)$ could simply be a concatenation of the bag-of-words features of the document $(\overline{\x})$,
one copy for each of the assigned labels in $\y$, $\overline{\x} \otimes \y$.
Several efficient inference algorithms have been developed to solve Equation~\eqref{eq:linearhyp}.

Consider the following stochastic family $\hypspace_{lin}$, parametrized by $\w$.
A hypothesis $\hp{\w}(\x) \in \hypspace_{lin}$ samples $\y$ from the distribution
\begin{equation*}
 \hp{\w}(\y \mid \x) = \exp(\w \cdot \phi(\x, \y)) / \mathbb{Z}(\x).
\end{equation*}
$\mathbb{Z}(x) = \sum_{\y' \in \Y} \exp(\w \cdot \phi(\x, \y')) $ is the partition function.
This can be thought of as the ``soft-max'' variant of the ``hard-max'' rules from Equation~\eqref{eq:linearhyp}.
Additionally, for a \emph{temperature} multiplier $\alpha > 1, \w \mapsto \alpha\w$ induces a more ``peaked'' distribution $\h_{\alpha \w}$
that preserves the modes of $\h_{\w}$, and intuitively is a ``more deterministic'' variant of $\h_{\w}$.

$\hp{\w}$ lies in the exponential family of distributions, and has a simple gradient,
\begin{equation*}
 \nabla \hp{\!\w\!}(\y \! \mid \! \x)\!=\!\hp{\!\w\!}(\y \! \mid \! \x)\! \left\{\! \phi(\x,\!\y)\!-\!\mathbb{E}_{\y'\!\sim\hp{\!\w\!}(\!\x\!)}\!\left[\phi(\x,\!\y')\right]\right\}.
\end{equation*}

Consider a bandit-feedback structured-output dataset $\sample = \{ (\x_1,\y_1,\loss_1,\ppty_1), \dots, (\x_\n,\y_\n,\loss_\n,\ppty_\n) \}$.
In multi-label document classification,
this data could be collected from an interactive labeling system,
where each $\y$ indicates the labels predicted by the system for a document $\x$.
The feedback $\loss(\x, \y)$ is how many labels (but not which ones) were correct.
To perform learning, first we scale the losses as outlined in Section~\ref{sec:loss_scale}.
Next, instantiating the CRM principle (Equation~\eqref{eq:crmobj}) for $\hypspace_{lin}$,
(using notation analogous to that in Theorem~\ref{thm:genbound}, adapted for $\hypspace_{lin}$), yields the \algnames\ training objective.

\paragraph{\algnames\ Training Objective:}
\begin{align}
\w^* &= \argmin_{\w \in \mathbb{R}^\N} \overline{u_{\w}} + \lambda \sqrt{\frac{\bm{Var}_{\w}(u)}{\n}} \label{eq:crmlinobj},
\end{align}
\begin{align*}
& {u_{\w}}^i \equiv \loss_i \min\{ M, \frac{\exp(\w \cdot \phi(\x_i, \y_i))}{\ppty_i \cdot \mathbb{Z}(\x_i)} \}, \,\, \overline{u_{\w}} \equiv \sum_{i=1}^\n {u_{\w}}^i / \n,\\
& \bm{Var}_{\w}(u) \equiv \sum_{i=1}^\n ({u_{\w}}^i - \overline{u_{\w}})^2 / (n - 1).
\end{align*}
While the objective in Equation~\eqref{eq:crmlinobj} is not convex in $\w$ 
(even for $\lambda = 0$),
prior work \cite{Yu2010}, \cite{Lewis2013} has established theoretically sound
modifications to L-BFGS for non-zmooth non-convex optimization.
We find that batch gradient descent (e.g., L-BFGS out of the box)
and the stochastic gradient approach introduced below find local optima that have good generalization error.

Software implementing \algnames\ is available at \url{http://www.cs.cornell.edu/~adith/poem/}
for download, as is all the code and data needed to run each of the experiments reported in Section~\ref{sec:experiments}.

\subsection{Iterated Variance Majorization}
\label{sec:sgd}
The \algnames\ training objective in Equation~\eqref{eq:crmlinobj}, specifically the variance term $\sqrt{{\bm{Var}_{\w}(u)}}$,
resists stochastic gradient optimization in the presented form. To remove this obstacle, we now develop
a Majorization-Minimization scheme, similar in spirit to recent approaches to multi-class SVMs \cite{vandenBurg2014} that can be shown to converge to a local optimum
of the \algnames\ training objective. 
In particular, we will show how to decompose $\sqrt{{\bm{Var}_{\w}(u)}}$ as
a sum of differentiable functions (e.g., $\sum_i {u_{\w}}^i$ or $\sum_i \{{u_{\w}}^i\}^2$)
so that we can optimize the overall training objective at scale using stochastic gradient descent.
\begin{proposition}
For any $\w_0$,
\begin{align*}
 \sqrt{{\bm{Var}_{\w}(u)}} &\le A_{\w_0} \sum_{i=1}^\n {u_{\w}}^i + B_{\w_0} \sum_{i=1}^\n \{{u_{\w}}^i\}^2 + C_{\w_0}\\
 &= Q(\w; \w_0).\\
 A_{\w_0} &\equiv -\overline{u_{\w_0}}/ \{ (n-1)\sqrt{\bm{Var}_{\w_0}(u)} \}, \\
 B_{\w_0} &\equiv 1/\{2(n-1)\sqrt{\bm{Var}_{\w_0}(u)}\}, \\
 C_{\w_0} &\equiv \frac{\n \{\overline{u_{\w_0}}\}^2}{2(n-1)\sqrt{\bm{Var}_{\w_0}(u)}} + \frac{\sqrt{\bm{Var}_{\w_0}(u)}}{2}.\\
\end{align*}
\vspace{-0.5in}
\end{proposition}
\begin{proof}
Consider a first order Taylor approximation of $\sqrt{{\bm{Var}_{\w}(u)}}$ around $\w_0$, $\sqrt{\cdot}$ is concave.
Again Taylor approximate $-\{\overline{u_{\w}}\}^2$, noting that $-\{ \cdot \}^2$ is concave.
\end{proof}
Iteratively minimizing $\w^{t+1} = \argmin_\w Q(\w; \w^{t})$
ensures that the sequence of iterates
$w^{1}, \dots, w^{t+1}$ are successive minimizers of $\sqrt{{\bm{Var}_{\w}(u)}}$.
Hence, during an epoch $t$, \algnames\ proceeds by sampling uniformly $i \sim \sample$,
computing ${u_{\w}}^i, \nabla {u_{\w}}^i$ and, for learning rate $\eta$,  updating
\begin{equation*}
\w \leftarrow \w - \eta \{ \nabla {u_{\w}}^i + \lambda \sqrt{\n} ( A_{w_t} \nabla {u_{\w}}^i + 2 B_{w_t} {u_{\w}}^i \nabla {u_{\w}}^i )\}.
\end{equation*}
After each epoch, $\w^{t+1} \leftarrow w$, and iterated minimization proceeds until convergence.

\section{Experiments}
\label{sec:experiments}
We now empirically evaluate the prediction performance and computational efficiency of \algnames.
Consider multi-label classification with input $\x \in \mathbb{R}^p$
and prediction $\y \in \{ 0, 1\}^q$.
Popular supervised algorithms that solve this problem
include Structured SVMs \cite{Tsochantaridis2004} and Conditional Random Fields \cite{Lafferty2001}.
In the simplest case, CRF essentially
performs logistic regression for each of the $q$ labels independently.
As outlined in Section~\ref{sec:blalgorithm},
we use a joint feature map: $\phi(\x, \y) = \x \otimes \y$.
We conducted experiments on different multi-label datasets
collected from the \href{http://www.csie.ntu.edu.tw/~cjlin/libsvmtools/datasets/multilabel.html}{LibSVM repository},
with different ranges for $p$ (features), $q$ (labels) and $\n$ (samples) represented as
summarized in Table~\ref{tab:datasets}.
\begin{table}[ht]
\vspace{-0.15in}
\caption{Corpus statistics for different multi-label datasets from the LibSVM repository.
LYRL was post-processed so that only top level categories were treated as labels.}
\label{tab:datasets}
\begin{center}
\begin{small}
\begin{tabular}{|l|c|c|c|c|}
\hline
Name & $p$(\# features)& $q$(\# labels) & $\n_{train}$ & $n_{test}$ \\
\hline
Scene	& 294 & 6 & 1211 & 1196 \\
Yeast	& 103 & 14 & 1500 & 917 \\
TMC	& 30438 & 22 & 21519 & 7077 \\
LYRL 	& 47236 & 4 & 23149 & 781265 \\
\hline
\end{tabular}
\end{small}
\end{center}
\vspace{-0.2in}
\end{table}
\paragraph{Experiment methodology.}
We employ the Supervised $\mapsto$ Bandit conversion \cite{Agarwal2014} method.
Here, we take a supervised dataset $\sample^* = \{ (\x_1, \ystar_1) \dots (\x_\n, \ystar_\n) \}$
and simulate a bandit feedback dataset from a logging policy $\piexp$ by sampling
$\y_i \sim \piexpp{\x_i}$ and collecting feedback $\Loss(\ystar_i, \y_i)$.
In principle, we could use any arbitrary stochastic policy as $\piexp$.
We choose a CRF trained on $5\%$ of $\sample^*$ as $\piexp$ using
default hyper-parameters,
since they provide probability distributions amenable to sampling.
In all the multi-label experiments, $\Loss(\ystar, \y)$ is the Hamming loss
between the supervised label $\ystar$ vs. the sampled label $\y$ for input $\x$.
Hamming loss is just the number of incorrectly assigned labels (both false positives and false negatives).
To create bandit feedback $\sample = \{ (\x_i,\y_i,\loss_i \equiv \Loss(\ystar_i, \y_i),\ppty_i \equiv \piexpp{\y_i \mid \x_i}) \}$,
we take four passes through $\sample^*$ and sample labels from $\piexp$.
Note that each supervised label is worth $\simeq \left| \Y \right| = 2^q$ bandit feedback labels.
We can explore different learning strategies (e.g., IPS, CRM, etc.)
on $\sample$ and obtain learnt weight vectors $\w_{ips}, \w_{crm}$, etc.
On the supervised test set, we then report the expected loss per instance
$\risk{w} = \frac{1}{\n_{test}} \sum_i \mathbb{E}_{y \sim \hp{\w}(\x_i)} \Loss(\ystar_i, \y)$
and compare the generalization performance of these learning strategies.
\paragraph{Baselines and learning methods.}
The expected Hamming loss of $\piexp$ is the baseline to beat.
Lower loss is better. 
The na\"{\i}ve, variance-agnostic approach to counterfactual learning \cite{Bottou2013} can be generalized to handle parametric multilabel classification
(Equation~\eqref{eq:crmlinobj} with $\lambda=0$).
We optimize it either using L-BFGS (IPS($\mathcal{B}$)) or stochastic optimization (IPS($\mathcal{S}$)).
\algnames($\mathcal{S}$) uses our Iterative-Majorization approach to variance regularization as outlined in Section~\ref{sec:sgd},
while \algnames($\mathcal{B}$) is a L-BFGS variant.
Finally, we report results from a supervised CRF as a skyline,
despite its unfair advantage of having access to the full-information examples.

We keep aside $25\%$ of $\sample$ as a validation set --
we use the unbiased counterfactual estimator from Equation~\eqref{eq:montecarlo}
for selecting hyper-parameters.
$\lambda = c \lambda^*$,
where $\lambda^*$ is the calibration factor from Section~\ref{sec:hyperparam_scale} and
$c \in \left[ 10^{-6}, \dots, 1 \right]$ in multiples of $10$.
The clipping constant $M$ is similarly set to the ratio of the $90\%ile$ to the $10\%ile$ propensity score observed in the
training set of $\sample$.
For all methods, when optimizing any objective over $\w$,
we always begin the optimization from $\w = 0 \ (\Rightarrow \h_\w = \textrm{ uniform}(\Y))$.
We use mini-batch AdaGrad \cite{Duchi2011} with batch size $= 100$
to adapt our learning rates for the stochastic approaches
and use progressive validation \cite{Blum1999} and gradient norms to detect convergence.
Finally, the entire experiment set-up is run 10 times (i.e. $\piexp$ trained
on randomly chosen $5\%$ subsets, $\sample$ re-created, and test set performance of different approaches collected)
and we report the averaged test set expected error across runs.
\subsection{Does variance regularization improve generalization?}
Results are reported in Table~\ref{tab:results}.
We statistically test the performance of \algnames\ against IPS
(batch variants are paired together, and the stochastic variants are paired together)
using a one-tailed paired difference t-test
at significance level of 0.05 across 10 runs of the experiment,
and find \algnames\ to be significantly better than IPS on each dataset and each optimization variant.
Furthermore, on all datasets \algnames\ learns a hypothesis that substantially improves over the performance of $\piexp$.
This suggests that the CRM principle is practically useful for designing learning algorithms,
and that the variance regularizer is indeed beneficial.
\begin{table}[ht]
\vspace{-0.15in}
\caption{Test set Hamming loss for different approaches to multi-label classification
on different datasets, averaged over 10 runs. 
\algnames\ is significantly better than IPS on each dataset and each optimization variant
(one-tailed paired difference t-test at significance level of 0.05).}
\label{tab:results}
\begin{center}
\begin{small}
\begin{tabular}{|l||c|c|c|c|c|}
\hline
 & Scene & Yeast & TMC & LYRL \\
\hline
$\piexp$ 		& 1.543 & 5.547 & 3.445 & 1.463 \\
\hline
IPS($\mathcal{B}$)	& 1.193 & 4.635 & 2.808 & 0.921\\
\algnames($\mathcal{B}$)& 1.168 & 4.480 & 2.197 & 0.918\\
  \hdashline
IPS($\mathcal{S}$)	& 1.519 & 4.614 & 3.023 & 1.118\\
\algnames($\mathcal{S}$)& 1.143 & 4.517 & 2.522 & 0.996\\
\hline
CRF			& 0.659 & 2.822 & 1.189 & 0.222 \\
\hline
\end{tabular}
\end{small}
\end{center}
\vspace{-0.15in}
\end{table}
\subsection{How computationally efficient is  \algnames?}
Table~\ref{tab:results_time} shows the time taken (in CPU seconds) to run each method on each dataset,
averaged over different validation runs when performing hyper-parameter grid search.
Some of the timing results are skewed by outliers, e.g., when under very weak regularization,
CRFs tend to take a lot longer to converge.
In aggregate, it is clear that the stochastic variants are able to recover good parameter settings
in a fraction of the time of batch L-BFGS optimization,
and this is even more pronounced when the number of labels grows 
(the run-time is dominated by computation of ${\mathbb{Z}(\x_i)}$).
\begin{table}[ht]
\vspace{-0.15in}
\caption{Average time in seconds for each validation run for different approaches to multi-label classification. CRF is the scikit-learn implementation \cite{Pedregosa2011}.
On all datasets, stochastic approaches are substantially faster than batch gradients.}
\label{tab:results_time}
\begin{center}
\begin{small}
\begin{tabular}{|l||c|c|c|c|c|}
\hline
 & Scene & Yeast & TMC & LYRL \\
\hline
IPS($\mathcal{B}$)   		& 2.58 & 47.61 & 136.34 & 21.01 \\
IPS($\mathcal{S}$)   		& 1.65 & 2.86 & 49.12 & 13.66 \\
\hdashline
\algnames($\mathcal{B}$) 	& 75.20 & 94.16 & 949.95 & 561.12 \\
\algnames($\mathcal{S}$)	& 4.71 & 5.02 & 276.13 & 120.09 \\
\hline
CRF				& 4.86 & 3.28 & 99.18 & 62.93 \\
\hline
\end{tabular}
\end{small}
\end{center}
\vspace{-0.1in}
\end{table}
\subsection{Can MAP predictions derived from stochastic policies perform well?}
For the policies learnt by \algnames\ as shown in Table~\ref{tab:results}, Table~\ref{tab:detmap} reports
the averaged performance of the deterministic predictor derived from them.
For a learnt weight vector $\w$, this simply amounts to applying Equation~(\ref{eq:linearhyp}).
In practice, this method of generating predictions can be substantially faster than sampling
since computing the $\argmax$ does not require computation of the partition function $\mathbb{Z}(\x)$
which can be expensive in structured output prediction.
From Table~\ref{tab:detmap}, we see that the loss of the deterministic predictor is typically not far from the loss of the stochastic policy,
but often slightly better.
\begin{table}[ht]
\vspace{-0.15in}
\caption{Mean Hamming loss of MAP predictions from the policies in Table~\ref{tab:results}.
\algnames$_{map}$ is not significantly \emph{worse} than \algnames\ (one-sided paired difference t-test, significance level 0.05).}
\label{tab:detmap}
\begin{center}
\begin{small}
\begin{tabular}{|l||c|c|c|c|c|}
\hline
 & Scene & Yeast & TMC & LYRL \\
\hline
\algnames($\mathcal{S}$)	& 1.143 & 4.517 & 2.522 & 0.996 \\
\algnames($\mathcal{S}$)$_{map}$& 1.143 & 4.065 & 2.299 & 0.880 \\
\hline
\end{tabular}
\end{small}
\end{center}
\vspace{-0.2in}
\end{table}
\subsection{How does generalization improve with size of $\sample$?}
\begin{figure}[ht]
\begin{center}
\begin{tikzpicture}
  \begin{semilogxaxis}[legend style={at={(0.1,0.4)},anchor=west}, log basis x=2, xlabel=$ReplayCount$,
  ylabel=$\risk{\w}$, xmin=1, xmax=256, samples at={0,...,8}, width=0.48\textwidth, height=0.3\textwidth]
  \addplot[mark=none, color=red, very thick] coordinates { (1,4.183) (2,4.183) (4,4.183) (8,4.183) (16,4.183) (32,4.183) (64,4.183) (128,4.183) (256,4.183)};
  \addplot[mark=none, color=blue, very thick] coordinates { (1,2.810) (2,2.810) (4,2.810) (8,2.810) (16,2.810) (32,2.810) (64,2.810) (128,2.810) (256,2.810)};
  \addplot coordinates { (1,4.080) (2,4.140) (4,3.982) (8,4.004) (16,3.863) (32,3.672) (64,3.515) (128,3.451) (256,3.294)};
  \legend{$\piexp$, CRF, \algnames($\mathcal{S}$)}
  \end{semilogxaxis}
\end{tikzpicture}
\vspace{-0.1in}
\caption{Generalization performance of \algnames($\mathcal{S}$) as a function of $\n$ on the Yeast dataset.
Even with $ReplayCount = 2^8$, \algnames($\mathcal{S}$) is learning from much less information
than the CRF (each supervised label conveys $2^{14}$ bandit label feedbacks).}
\label{fig:learningrate}
\end{center}
\vspace{-0.1in}
\end{figure}
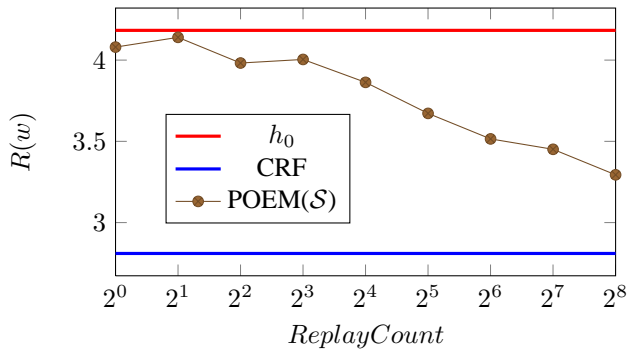
As we collect more data under $\piexp$, our generalization error bound indicates that prediction performance
should eventually approach that of the optimal hypothesis in the hypothesis space. 
We can simulate $\n \rightarrow \infty$ by replaying
the training data multiple times, collecting samples $\y \sim \piexpp{\x}$.
In the limit, we would observe every possible $\y$ in the bandit feedback dataset,
since $\piexpp{\x}$ has non-zero probability of exploring each prediction $\y$.
However, the learning rate may be slow, since the exponential model family
has very thin tails, and hence may not be an ideal logging distribution to learn from.
Holding all other details of the experiment setup fixed, we vary the number of times we replayed the training set ($ReplayCount$)
to collect samples from $\piexp$,
and report the performance of \algnames($\mathcal{S}$) on the Yeast dataset in Figure~\ref{fig:learningrate}.
\subsection{How does quality of $\piexp$ affect learning?}
In this experiment, we change the fraction of the training set $f \cdot \n_{train}$ that was used to train
the logging policy; as $f$ is increased, the quality of $\piexp$ improves.
Intuitively, there's a trade-off: better $\piexp$ probably samples correct predictions more often
and so produces a higher quality $\sample$ to learn from, but it should also be harder to beat $\piexp$.
\begin{figure}[ht]
\begin{center}
\begin{tikzpicture}
  \begin{axis}[legend style={at={(0.6,0.75)},anchor=west}, xlabel=$f$,
  ylabel=$\risk{\w}$, xmin=0.01, xmax=1.0, width=0.48\textwidth, height=0.3\textwidth]
  \addplot coordinates { (0.01,6.530) (0.05,5.520) (0.1,5.058) (0.2,4.786) (0.35,4.605) (0.5,4.502) (0.75,4.405) (1.0,4.344)};
  \addplot coordinates { (0.01,5.560) (0.05,4.639) (0.1,4.360) (0.2,4.317) (0.35,4.016) (0.5,4.133) (0.75,4.071) (1.0,4.105)};
  \legend{$\piexp$, \algnames($\mathcal{S}$)}
  \end{axis}
\end{tikzpicture}
\vspace{-0.1in}
\caption{Performance of \algnames($\mathcal{S}$) on the Yeast dataset as $\piexp$ is improved.
The fraction $f$ of the supervised training set used to train $\piexp$
is varied to control $\piexp$'s quality.
$\piexp$ performance does not reach CRF when $f=1$
because we do not tune hyper-parameters, and we report its expected loss, not the loss of its MAP prediction.}
\label{fig:logging_quality}
\end{center}
\vspace{-0.1in}
\end{figure}
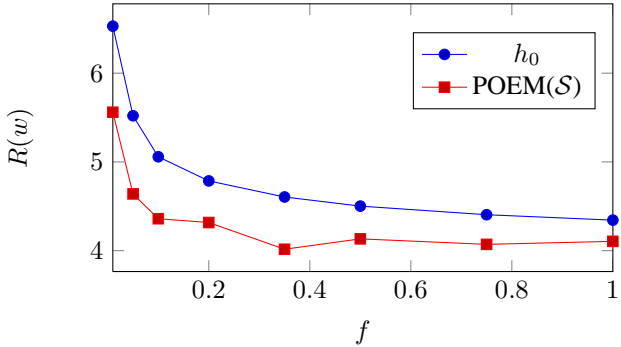
We vary $f$ from $1\%$ to $100\%$ while keeping all other conditions identical to the original experiment setup
in Figure~\ref{fig:logging_quality},
and find that \algnames($\mathcal{S}$) is able to consistently find a hypothesis at least as good as $\piexp$.
Moreover, even $\sample$ collected from a poor quality $\piexp$ ($0.5 \le f \le 0.2$) allows \algnames($\mathcal{S}$) to
effectively learn an improved policy.
 \subsection{How does stochasticity of $\piexp$ affect learning?}
\label{sec:stochastic_piexp}
Finally, the theory suggests that counterfactual learning is only possible when $\piexp$
is sufficiently stochastic (the generalization bounds hold with high probability in the samples drawn from $\piexp$).
Does CRM degrade gracefully when this assumption is violated?
We test this by introducing the \emph{temperature} multiplier $\w \mapsto \alpha \w, \alpha > 0$
(as discussed in Section~\ref{sec:blalgorithm}) into the logging policy.
For $\piexp = \hp{\w_0}$, we scale $\w_0 \mapsto \alpha \w_0$, to derive a ``more deterministic'' variant of $\piexp$,
and generate $\sample \sim \hp{\alpha \w_0}$.
We report the performance of \algnames($\mathcal{S}$)
on the LYRL dataset in Figure~\ref{fig:logging_stochastic} as we change $\alpha \in \left[0.5, \dots, 32\right]$,
compared against $\piexp$, and the deterministic predictor -- $\piexp\ {map}$ -- derived from $\piexp$.
So long as there is some minimum amount of stochasticity in $\piexp$,
\algnames($\mathcal{S}$) is still able to find a $\w$ that improves upon $\piexp$ and $\piexp\ {map}$.
The margin of improvement is typically greater when $\piexp$ is more stochastic.
Even when $\piexp$ is too deterministic ($\alpha \ge 2^4$),
performance of \algnames($\mathcal{S}$) simply recovers $\piexp \ {map}$,
suggesting that the CRM principle indeed achieves robust learning.
\begin{figure}[ht]
\begin{center}
\begin{tikzpicture}
  \begin{semilogxaxis}[legend style={at={(0.6,0.7)},anchor=west}, log basis x=2, xlabel=$\alpha$,
  ylabel=$\risk{\w}$, xmin=0.5, xmax=32,width=0.48\textwidth, height=0.3\textwidth]
  \addplot coordinates { (0.5,1.702) (1,1.464) (2,1.197) (4,1.050) (8,0.994) (16,0.964) (32,0.958)};
  \addplot coordinates { (0.5,0.976) (1,0.944) (2,0.932) (4,0.902) (8,0.899) (16,0.916) (32,0.984)};
  \addplot coordinates { (0.5,0.960) (1,0.960) (2,0.960) (4,0.960) (8,0.960) (16,0.960) (32,0.960)};
  \legend{$\piexp$, \algnames($\mathcal{S}$), $\piexp \ {map}$}
  \end{semilogxaxis}
\end{tikzpicture}
\vspace{-0.2in}
\caption{Performance of \algnames($\mathcal{S}$) on the LYRL dataset as $\piexp$ becomes more deterministic.
For $\alpha \ge 2^5$, $\piexp \equiv \piexp\ {map}$ (within machine precision).}
\label{fig:logging_stochastic}
\end{center}
\vspace{-0.2in}
\end{figure}
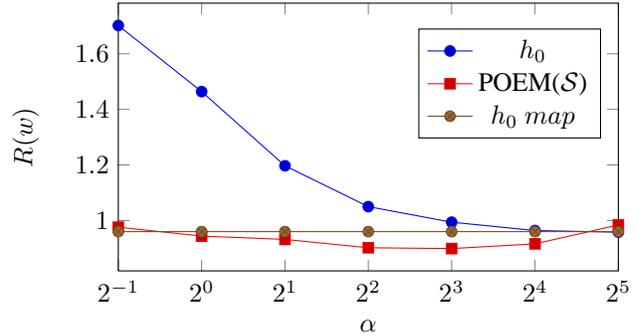

We observe the same trends (Figures~\ref{fig:learningrate}, \ref{fig:logging_quality} and \ref{fig:logging_stochastic}) across all datasets and optimization variants.
They also remain unchanged when we include $l2-$regularization (analogous to
supervised CRFs to capture the capacity of $\hypspace_{lin}$).

\section{Conclusion}
\label{sec:conclusion}
Counterfactual risk minimization serves as a robust principle
to design algorithms that can learn from a batch of bandit feedback interactions.
The key insight for CRM is to expand the classical notion of a hypothesis class to include stochastic policies,
reason about variance in the risk estimator,
and derive a generalization error bound over this hypothesis space.
The practical take-away is a simple, data-dependent regularizer that guarantees robust learning.
Following the CRM principle, we developed \algnames\ for structured output prediction.
\algnames\ can optimize over rich policy families
(exponential models corresponding to linear rules in supervised learning),
and deal with massive output spaces as efficiently as classical supervised methods.

The CRM principle more generally applies to supervised learning with non-differentiable losses,
since the objective does not require the gradient of the loss function.
We also foresee extensions of this work that relax some of the assumptions,
e.g., to handle noisy $\loss(\cdot, \cdot)$, and ordinal or co-active feedback, or adaptive $\piexp$ etc.

\section*{Acknowledgement} \label{sec:acknowledgement}
This research was funded in part through NSF Awards IIS-1247637 and IIS-1217686,
the JTCII Cornell-Technion Research Fund, and a gift from Bloomberg.
We thank Chenhao Tan, Karthik Raman and Vikram Rao for proofreading our manuscript,
and the anonymous reviewers of ICML for their constructive feedback.

\bibliography{PRM}
\bibliographystyle{icml2015}

\end{document}